%% file: main.tex
\documentclass[conference]{IEEEtran}

\IEEEoverridecommandlockouts                              
\usepackage{amsmath}
\usepackage{amssymb}  
\usepackage{color}
\usepackage{standalone}
\usepackage{tikz}
\usepackage{subfigure}
\usepackage{graphics}
\usepackage{epsfig}

\newif\iftextonly
\newif\ifdual 

\iftextonly
\usepackage{comment}
\excludecomment{figure}

\excludecomment{equation}

\excludecomment{align}

\fi

\newtheorem{theorem}{Theorem}
\newtheorem{lemma}{Lemma}

\newcommand{\nodeweight}{\nu}
\newcommand{\edgeweight}{\omega}
\newcommand{\logweight}{\edgeweight_O}
\newcommand{\feasible}{\mathcal{X}(p_s,\edgeweight)}

\newcommand{\objective}[1]{\sum_{j=1}^Vd_j \mathbb{E}\left[x_j(#1)\right]}
\newcommand{\objectiveGrad}[2]{\Delta J(#1\mid#2)}

\newcommand{\Xk}[1]{\{\rho_k\}_{k=1}^{#1}}

\newcommand{\approxfact}{1-e^{-p_s/\lambda}}
\newcommand{\problemname}{ TSO }
\usepackage{dsfont}
\usepackage{algorithm}
\usepackage{algpseudocode}

\renewcommand{\baselinestretch}{.955}

\title{\LARGE \bf
The Team Surviving Orienteers Problem: Routing Robots\\ in Uncertain Environments with Survival Constraints
}

\author{Stefan Jorgensen,  Robert H. Chen, Mark B. Milam, and Marco Pavone %
\thanks{Stefan Jorgensen is with the Department of Electrical Engineering, Stanford University, Stanford, California 94305. His work is supported by NSF grant DGE-114747 {\tt\footnotesize stefantj@stanford.edu}

Robert H. Chen and Mark B. Milam are with NG Next, Northrop Grumman Aerospace Systems, Redondo Beach, California 90278 {\tt \footnotesize \{robert.chen,mark.milam\}@ngc.com}

Marco Pavone is with the Department of Aeronautics \& Astronautics, Stanford University, Stanford, California 94035  {\tt\footnotesize pavone@stanford.edu}}

}         

       \newcommand{\mpmargin}[2]{{\color{cyan}#1}\marginpar{\color{cyan}\raggedright\footnotesize [MP]:
#2}}

\begin{document}
\maketitle
\thispagestyle{empty}
\pagestyle{empty} 

\begin{abstract}
In this paper we study the following multi-robot coordination problem: given a graph, where each edge is weighted by the probability of surviving while traversing it, find a set of paths for $K$ robots that maximizes the expected number of nodes collectively visited, subject to constraints on the probability that each robot survives to its destination. We call this problem the Team Surviving Orienteers (TSO) problem. The TSO problem is motivated by scenarios where a team of robots must traverse a dangerous, uncertain environment, such as aid delivery in disaster or war zones. We present the TSO problem formally along with several variants, which represent ``survivability-aware" counterparts for  a wide range of multi-robot coordination problems such as vehicle routing, patrolling, and informative path planning. We propose an approximate greedy approach for selecting paths, and prove that the value of its output is bounded within a factor $\approxfact$ of the optimum where $p_s$ is the per-robot survival probability threshold, and $1/\lambda\leq 1$ is the approximation factor of an oracle routine for the well-known orienteering problem. Our approach has linear time complexity in the team size and polynomial complexity in the graph size. Using numerical simulations, we verify that our approach is close to the optimum in practice and that it scales to problems with hundreds of nodes and tens of robots.
\end{abstract}

\section{Introduction}
Consider the problem of delivering humanitarian aid in a war zone with a team of robots. There are a number of sites which need the resources, but traveling among these sites is dangerous. While the aid agency wants to deliver aid to every city, it also seeks to limit the number of assets that are lost. We formalize this problem as a generalization of the orienteering problem \cite{BLG-LL-RV:87}, whereby one seeks to visit as many nodes in a graph as possible given a budget constraint and travel costs. In the aid delivery case, the travel costs are the probability that a robotic aid vehicle is lost while traveling between sites, and the goal is to maximize the {\em expected} number of sites visited by the vehicles, while keeping the return {\em probability} for each vehicle above a specified survival threshold (i.e., while fulfilling a chance constraint for the survival of each vehicle). We refer to such problem formulation as the ``team surviving orienteers" (TSO) problem, illustrated in Figure \ref{fig:aid_illustration}. The TSO problem is distinct from previous work because of its notion of \emph{risky traversal:} when a robot traverses an edge, there is a probability that it is lost and does not visit any other nodes. This creates a complex, history-dependent coupling between the edges chosen and the distribution of nodes visited, which precludes the application of existing approaches available for the traditional orienteering problem. 

The objective of this paper is to devise a constant-factor approximation algorithm for the TSO problem. Our key technical insight is that the expected number of nodes visited satisfies a diminishing returns property known as submodularity, which for set functions means that $f(A \cup B) \le f(A) + f(B)$. We develop a {\em linearization procedure} for the problem, which leads to a greedy algorithm that enjoys a constant-factor approximation guarantee. We emphasize that while a number of works have considered orienteering problems with submodular objectives \cite{CC-MP:05,AMC-MG-BWT:11,HZ-YV:16} or chance constraints \cite{AG-RK-VN-RR:12,PV-AK:13} separately, the combination of the two makes the TSO a novel problem, as detailed next.

\begin{figure}[t]
    \centering
    \includegraphics[width=.4\textwidth]{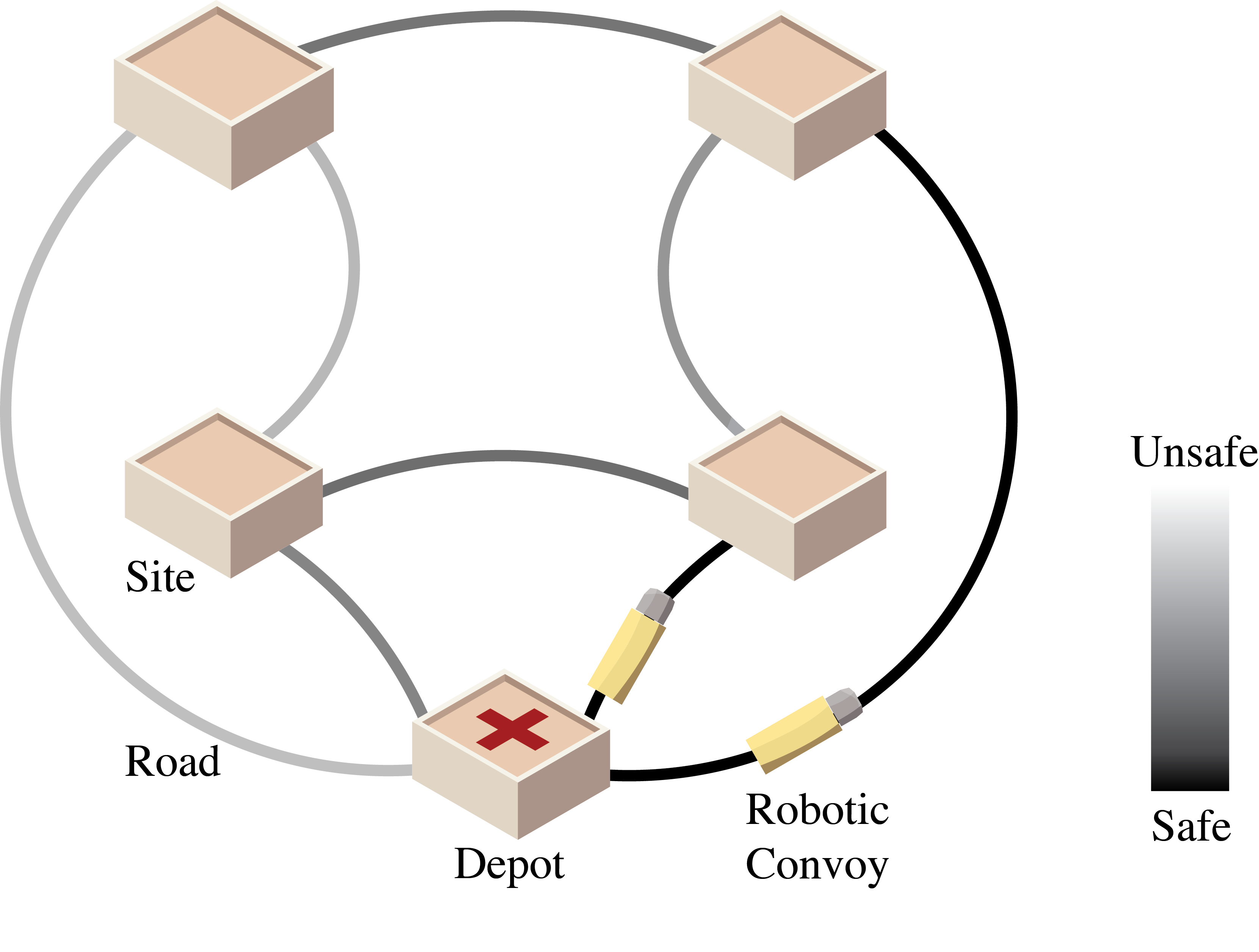}
    \caption{Illustration of the TSO problem applied to an aid delivery scenario. The objective is to maximize the expected number of sites visited by at least one robotic convoy. Travel between sites is risky (as emphasized by the gray color scale for each edge), and paths must be planned to ensure that the return probability for each vehicle is above a survival threshold. }
    \label{fig:aid_illustration}
\end{figure}

\emph{Related work.} The orienteering problem (OP) has been extensively studied \cite{PV-WS-DVO:11,AG-HCL-PV:16} and is known to be NP-hard. Over the past decade a number of constant-factor approximation algorithms have been developed for special cases of the problem \cite{CC-NK-MP:12}.  Below we highlight several variants which share either similar objectives or constraints as the TSO problem. 

The \emph{submodular orienteering problem} considers finding a single path which maximizes a submodular reward function of the nodes visited. The recursive greedy algorithm proposed in \cite{CC-MP:05} yields a solution in quasi-polynomial time with reward lower bounded as $\Omega(\text{OPT}/\log(\text{OPT}))$, where $\text{OPT}$ is the optimum value. More recently, \cite{HZ-YV:16} develops a (polynomial time) generalized cost-benefit algorithm, useful when searching the feasible set is NP-hard (such as longest path problems). The authors show that the output of their algorithm is $\Omega(\frac{1}{2}(1-1/e)\text{OPT}^*)$, where $\text{OPT}^*$ is the optimum for a relaxed problem. In our context, $\text{OPT}^*$ roughly corresponds to the maximum expected number of nodes visited with survival probability constraint $\sqrt{p_s}$, which may be significantly different from the actual optimum. Our work considers a specific submodular function, however we incorporate risky traversal, give a stronger (problem independent) guarantees, and discuss an extension to general submodular functions. 
In the \emph{orienteering problem with stochastic travel times} proposed by \cite{AMC-MG-BWT:11}, travel times are stochastic and reward is accumulated at a node only if it is visited before a deadline. This setting could be used to solve the single robot special case of the TSO by using a log transformation on the survival probabilities, but \cite{AMC-MG-BWT:11} does not provide any polynomial time guarantees. 
In the \emph{risk-sensitive orienteering problem} \cite{PV-AK:13}, the goal is to maximize the sum of rewards (which is history independent) subject to a constraint on the probability that that the path cost is large. The TSO unifies the models of the risk-sensitive and stochastic travel time variants by considering both a submodular objective (expected number of nodes visited) and a chance constraint on the total cost. Furthermore, we provide a constant-factor guarantee for the \emph{team} version of this problem.

A second closely-related area of research is represented by the vehicle routing problem (VRP) \cite{VP-MG-CG-ALM:13,HNP-MW-CAK:15}, which is a family of problems focused on finding a set of paths that maximize quality of service subject to budget or time constraints. The {\em probabilistic VRP} (PVRP) considers stochastic edge costs with chance constraints on the path costs -- similar to the risk-averse orienteering and the TSO problem constraints. The authors of \cite{GL-FL-HM:89} pose the simultaneous location-routing problem, where both routes and depot locations are selected to minimize path costs subject to  a probabilistic connectivity constraint, which specifies the {average case risk} rather than individual risks. More general settings were considered in \cite{BLG-JRY:92}, which considers several distribution families (such as the exponential and normal distributions), and \cite{WRS-BLG:83}, which considers nonlinear risk constraints. In contrast to the TSO problem, the PVRP requires {\em every} node to be visited and seeks to minimize the travel cost. In the TSO problem, we require every path to be safe and maximize the expected number of nodes visited.

A third related branch of literature is the informative path planning problem (IPP), which seeks to find a set of $K$ paths for mobile robotic sensors  in order to maximize the information gained about an environment. One of the earliest IPP approaches \cite{AS-AK-CG-WJK:09} extends the recursive greedy algorithm of \cite{CC-MP:05} using a spatial decomposition to generate paths for multiple robots. They use submodularity of information gain to provide performance guarantees. Sampling-based approaches to IPP were proposed by \cite{GAH-GSS:13}, which come with asymptotic guarantees on optimality. The structure of the IPP is most similar to that of the TSO problem (since it is a multi-agent path planning problem with a submodular objective function which is nonlinear and history dependent), but it does not capture the notion of risky traversal which is essential to the TSO. Our general approach is inspired by works such as \cite{NA-JLN-KD-GJP:15}, but for the TSO problem we are able to further exploit the problem structure to derive constant-factor guarantees for our polynomial time algorithm.

\emph{Statement of Contributions. }
The contribution of this paper is fourfold. First, we propose a generalization of the orienteering problem, referred to as the TSO problem. By considering a multi-robot (team) setting, we extend the state of the art for the submodular orienteering problem, and by maximizing the expected number of nodes visited at least once, we extend the state of the art in the probabilistic vehicle routing literature. From a practical standpoint, as discussed in Section \ref{sec:problem_statement}, the TSO problem represents a ``survivability-aware" counterpart for  a wide range of multi-robot coordination problems such as vehicle routing, patrolling, and informative path planning. Second, we establish that the objective function of the TSO problem is submodular, provide a linear relaxation of the single robot TSO problem (which can be solved as a standard orienteering problem), and show that the solution to the relaxed problem provides a close approximation of the optimal solution of the single robot TSO problem. Third, we propose an approximate greedy algorithm which has polynomial complexity in the number of nodes and linear complexity in the team size, and prove that the value of the output of our algorithm is $\Omega( (\approxfact) \text{OPT})$, where $\text{OPT}$ is the optimum value, $p_s$ is the per-robot survival probability constraint, and $1/\lambda\leq 1$ is the approximation factor of an oracle routine for the solution to the orienteering problem (we note that, in practice, $p_s$ is usually close to unity). Finally, we demonstrate the effectiveness of our algorithm for large problems using simulations by solving a problem with 900 nodes and 25 robots.

\emph{Organization.} In Section \ref{sec:background} we review key background information. In Section \ref{sec:problem_statement} we state the problem formally, give an example, and describe several variants and applications of the TSO. In Section \ref{sec:solution} we show that the objective function is submodular and describe the linear relaxation technique. We then demonstrate how to solve the relaxed problem as an orienteering problem, outline a greedy solution approach for the TSO problem, give approximation guarantees, and characterize the algorithm's complexity. We finally give extensions of the algorithm for variants of the TSO. In Section \ref{sec:numbers} we verify the performance bounds and demonstrate the scalability of our approach. Finally, we outline future work and draw conclusions in Section \ref{sec:conclusion}.

\section{Background}\label{sec:background}
In this section we review key material for our work and extend a well-known theorem in the combinatorial optimization literature to our setting.

\subsection{Submodularity} Submodularity is the property of `diminishing returns' for set functions. The following definitions are summarized from \cite{AK-DG:12}. Given a set $\mathcal{X}$, its possible subsets are represented by $2^\mathcal{X}$. For two sets $X$ and $X'$, the set $X' \setminus X$ contains all elements in $X'$ but not $X$.  A set function $f: 2^\mathcal{X} \to \mathbb{R}$ is said to be \emph{normalized} if $f(\emptyset) = 0$ and to be \emph{monotone} if for every $X \subseteq X'\subseteq \mathcal{X}$, $f(X) \le f(X')$. A set function $f: 2^\mathcal{X}\to \mathbb{R}$ is \emph{submodular} if for every $X \subseteq X' \subset \mathcal{X}$, $x \in \mathcal{X} \setminus X'$, we have 
\begin{equation*}
 f(X \cup \{x\}) - f(X) \ge f(X' \cup \{x\}) - f(X').
\end{equation*}
 The quantity on the left hand side is the \emph{discrete derivative} of $f$ at $X$ with respect to $x$, which we write as $\Delta f(x\mid X)$.

\subsection{The Approximate Greedy Algorithm}\label{sec:BG:approx}
A typical submodular maximization problem entails finding a set $X \subseteq \mathcal{X}$ with cardinality $K$ that maximizes $f$. Finding an optimal solution, $X^*$, is NP-hard for general submodular functions \cite{AK-DG:12}. The \emph{greedy algorithm} constructs a set $\bar{X}_K=\{x_1,\dots,x_K\}$ by iteratively adding an element $x$ which maximizes the discrete derivative of $f$ at the partial set already selected. In other words the $q$th element satisfies:
$$x_{q} \in \underset{x\in\mathcal{X}\setminus\bar{X}_{q-1}}{\text{argmax}}\ \Delta f(x \mid \bar{X}_{q-1}). $$
We refer to the optimization problem above as `the greedy sub-problem' at iteration $q$. A well-known theorem from \cite{GLN-LAW-MLF:78} states that if $f$ is a monotone, normalized, non-negative, and submodular function, then $f(\bar{X}_K)\ge (1-\frac{1}{e})f(X^*)$. This is a powerful result, but if the set $\mathcal{X}$ is large we might only be able to approximately maximize the discrete derivative. An $\alpha$-approximate greedy algorithm constructs the set $\hat{X}_K$ by iteratively adding elements which {\em approximately} maximize the discrete derivative. In particular for some fixed $\alpha \le 1$, the $q$th element $\hat{x}_q$ satisfies:
$$ \Delta f(\hat{x}_q \mid \hat{X}_{q-1}) \ge \alpha \Delta f(x\mid \hat{X}_{q-1}) \qquad \forall x \in \mathcal{X}\setminus \hat{X}_{q-1}.$$
In the following theorem, we extend Theorem 4.2 of \cite{GLN-LAW-MLF:78} for the $\alpha$-approximate greedy algorithm:
\begin{theorem}[$\alpha$-approximate greedy guarantee]\label{thm:extension}
Let $f$ be a monotone, normalized, non-negative, and submodular function with discrete derivative $\Delta f$. Then for the output of any $\alpha$-approximate greedy algorithm with $L$ elements, $\hat{X}_L$, we have the following inequality:
$$f(\hat{X}_L) \ge \left( 1 - e^{-\alpha L/K}\right) \max_{X \in 2^{\mathcal{X}} : \lvert X\rvert = K} f(X).$$ 
\end{theorem}
\begin{proof}
The case where $L=K$ is a special case of Theorem 1 from \cite{KW-RI-JB:14}. To generalize to $L \ge K$ we extend the proof for the greedy algorithm in \cite{AK-DG:12}. Let $X^* \in 2^\mathcal{X}$ be the set which maximizes $f(X)$ subject to the cardinality constraint $\lvert X \rvert = K$. For $\ell < L$, we have:
\begin{align*}
f(X^*) &\le f(X^* \cup \hat{X}_{\ell})\\
&= f(\hat{X}_{\ell}) + \sum_{k=1}^K \Delta f(x_k^* \mid \hat{X}_{\ell} \cup \{x_1^*,\dots,x_{k-1}^*\})\\
&\le f(\hat{X}_{\ell}) + \sum_{k=1}^K \Delta f(x_k^* \mid \hat{X}_\ell)\\
&\le f(\hat{X}_{\ell}) + \frac{1}{\alpha} \sum_{k=1}^K \Delta f(\hat{x}_{\ell+1}\mid \hat{X}_\ell)\\
&\le f(\hat{X}_{\ell}) + \frac{K}{\alpha} (f(\hat{X}_{\ell+1}) - f(\hat{X}_{\ell})).
\end{align*}
The first line follows from the monotonicity of $f$, the second is a telescoping sum, and the third follows from the submodularity of $f$. The fourth line is due to the $\alpha$-approximate greedy construction of $\hat{X}_L$, and the last is because $\lvert X^*\rvert \le K$. Now define $\delta_\ell = f(X^*) - f(\hat{X}_\ell)$. We can re-arrange the inequality above to yield:
\begin{equation*}
\delta_{\ell+1} \le \left(1 - \frac{\alpha}{K}\right) \delta_{\ell} \le \left(1 - \frac{\alpha}{K}\right)^{\ell+1}\delta_0.
\end{equation*}
Since $f$ is non-negative, $\delta_0 \le f(X^*)$ and using the inequality $1-x \le e^{-x}$ we get
\begin{equation*}
\delta_L \le \left(1-\frac{\alpha}{K} \right)^L \delta_0 \le \left(e^{-\alpha L/K} \right)f(X^*).
\end{equation*}
Now substituting $\delta_L = f(X^*) - f(\hat{X}_L)$ and rearranging:
\begin{equation*}
f(\hat{X}_L) \ge \left( 1 - e^{-\alpha L/K}\right) f(X^*).
\end{equation*}
\end{proof}


\subsection{Graphs}
Let $\mathcal{G}(\mathcal{V},\mathcal{E})$ denote a graph, where $\mathcal{V}$ is the node set and $\mathcal{E}\subset \mathcal{V} \times \mathcal{V}$ is the edge set. Explicitly, an edge is an ordered pair of nodes $(i,j)$, and represents the ability to travel from the \emph{source node} $i$ to the \emph{sink node} $j$. A graph is called \emph{simple} if there is only one edge which connects any given pair of nodes.  A path is an ordered sequence of unique nodes such that there is an edge between adjacent nodes. For $n\ge 0$, we denote the $n$th node in path $\rho$ by $\rho(n)$ and denote the number of edges in a path by $\lvert \rho \rvert$.

\section{Problem statement}\label{sec:problem_statement}
In this section we give the formal problem statement for the TSO, work out an example problem, and describe applications and variants of the problem.
\subsection{Formal Problem Description}\label{sec:prob_description}
Let $\mathcal{G}$ be a simple graph with $\lvert\mathcal{V}\rvert = V$ nodes. Edge weights $\edgeweight:\mathcal{E} \to (0,1]$ correspond to the probability of survival for traversing an edge. 
At step $n$ a robot following path $\rho$ traverses edge $e_\rho^n=(\rho(n-1),\, \rho(n))$. Define the independent Bernoulli random variables $s_n(\rho)$ which are $1$ with probability $\edgeweight(e_\rho^n)$ and $0$ with probability $1-\edgeweight(e_\rho^n)$. If a robot follows path $\rho$, the random variables $a_n(\rho) = \prod_{i=1}^n s_i(\rho)$ can be interpreted as being 1 if the robot `survived' all of the edges taken until step $n$ and 0 if the robot `fails' on or before step $n$.

Given a start node $v_s$, a terminal node $v_t$, and survival probability $p_s$ we must find $K \geq 1$ paths $\{\rho_k\}_{k=1}^K$ (one for each of $K$ robots) such that, for all $k$, the probability that $a_{\lvert\rho_k\rvert}(\rho_k) = 1$ is at least $p_s$, $\rho_k(0) = v_s$ and $\rho_k(\lvert\rho_k\rvert) = v_t$. The set of paths which satisfy these constraints is written as $\feasible$. Using Dijkstra's algorithm one can readily test whether $\feasible$ is empty as follows: For each node $j$, set edge weights as $-\log(\edgeweight(e))$, compute the shortest path from $v_s$ to $j$, then delete the edges in that path and compute the shortest path from $j$ to $v_t$. If the sum of edge weights along both paths is less than $-\log(p_s)$ then the node is reachable, otherwise it is not. This approach can prove whether $\feasible$ is empty after $O(V^2\log(V))$ computations. From here on we assume that $\feasible$ is non-empty.

\begin{figure}[t]
    \centering
    \includegraphics{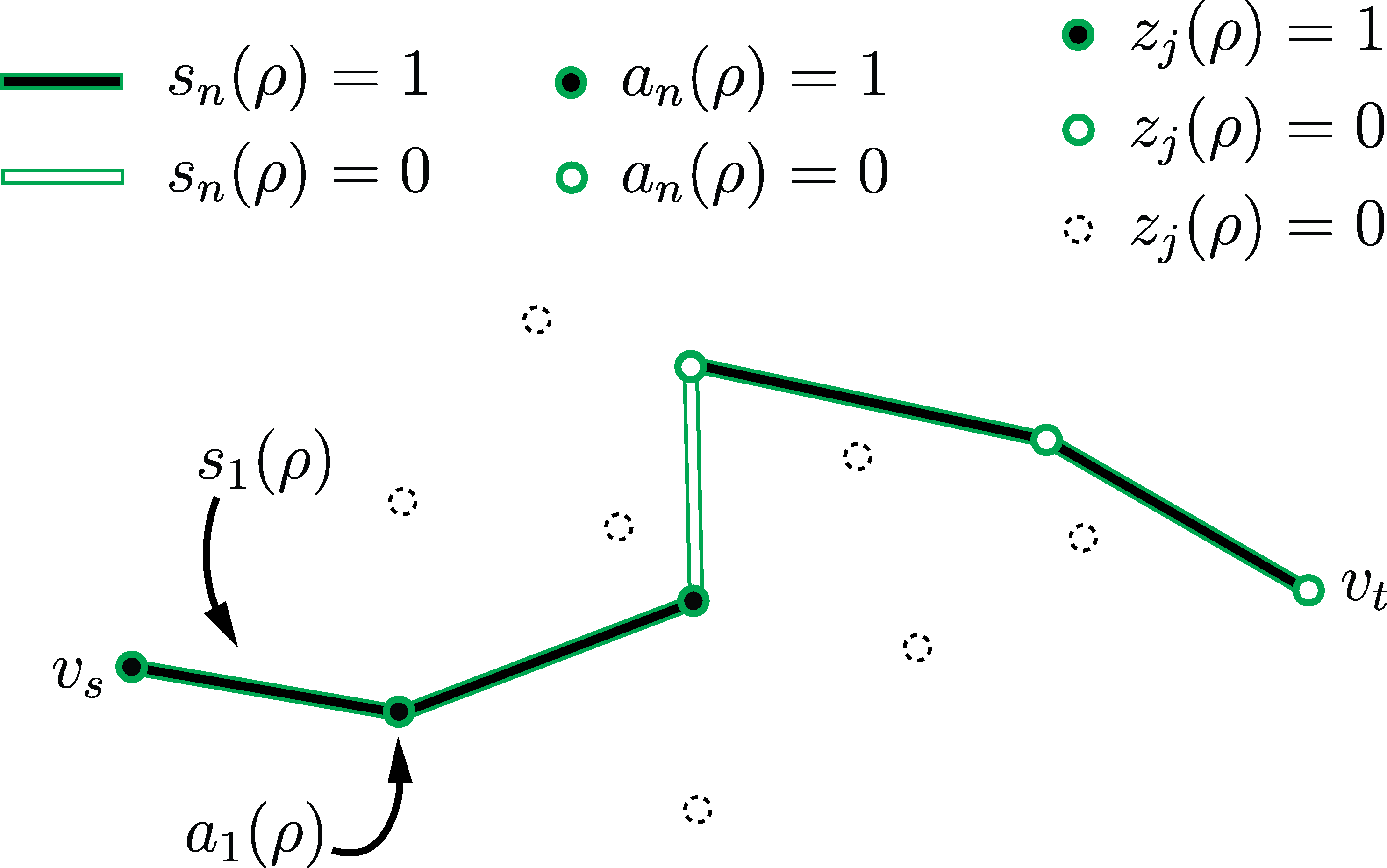}
    \caption{Illustration of the notation used. The robot plans to take path $\rho$, whose edges are represented by lines. The fill of the lines represent the value of $s_n(\rho)$. In this example $s_3(\rho) = 0$, which means that $a_3(\rho) = a_4(\rho) = a_5(\rho) = 0$. The variables $z_j(\rho)$ are zero if either the robot fails before reaching node $j$ or if node $j$ is not on the planned path.}
    \label{fig:notation}
\end{figure}

Define the indicator function $\mathbb{I}\{x\}$, which is 1 if $x$ is true (or nonzero) and zero otherwise. Define the Bernoulli random variables for $j=1,\dots,V$:
\begin{equation*}
z_j(\rho) := \max_{n=1,\dots,\lvert \rho \rvert} \, a_n(\rho)\,\mathbb{I}\{\rho(n)=j\},
\end{equation*}
which are 1 if a robot following path $\rho$ visits node $j$ and 0 otherwise. Because $z_j(\rho)$ is independent of $z_j(\rho')$ for $\rho \neq \rho'$, the event that node $j$ is visited by at least one robot is:
	\begin{equation*}
x_j\left(\{\rho_k\}_{k=1}^K\right) := 1-\prod_{k=1}^K (1-z_j(\rho_k)),
\end{equation*}
and the total number of nodes visited is the sum of these variables over $j$. Let $d_j > 0$ be the priority of visiting node $j$. Then the TSO problem is formally stated as:


\begin{quote}{\bf Team Surviving Orienteers Problem:}
Given a graph $\mathcal{G}$, edge weights $\edgeweight$, survival probability constraint $p_s$ and team size $K$, maximize the weighted expected number of nodes visited by at least one robot:
\begin{equation*}
\begin{aligned}
& \underset{\rho_1,\dots,\rho_K}{\text{maximize}} & & \sum_{j=1}^Vd_j\mathbb{E}\left[x_j\left(\{\rho_k\}_{k=1}^K\right)\right]  & &\\
& \text{subject to} & &  \mathbb{P}\{a_{\lvert\rho_k\rvert}(\rho_k) = 1\} \ge p_s & k=1,\dots,K\\
& & & \rho_k(0) = v_s & k=1,\dots,K\\
& & & \rho_k(\lvert\rho_k\rvert) = v_t & k=1,\dots,K
\end{aligned}
\end{equation*}
\end{quote}

The objective is the weighted expected number of nodes visited by the $K$ robots. The first set of constraints enforces the survival probability, the second and third sets of constraints enforce the initial and final node constraints. 

\subsection{Example}

An example of the \problemname problem is given in Figure \ref{fig:examplegraph}. There are five nodes, and edge weights are shown next to their respective edges. Two robots start at node 1, and must end at node 1 with probability at least $p_s=0.75$.  Path $\rho_1 = \{1,3,5,2,1\}$ is shown in Figure \ref{fig:exampleone}, and path $\rho_2 = \{1,4,5,2,1\}$ is shown alongside $\rho_1$ in Figure \ref{fig:exampletwo}. Robot 1 visits node 3 with probability 1.0 and node 5 with probability 0.96. Robot 2 also visits node 5 with probability 0.96 and so the probability at least one robot visits node 5 is $\mathbb{E}\left[x_5(\{\rho_1,\rho_2\})\right]= 0.9984$. The probability that robot 1 returns safely is $\mathbb{E}\left[a_{\lvert\rho_1\rvert}(\rho_1)\right] = 0.794$. For this simple problem, $\rho_1$ and $\rho_2$ are two of three possible paths  (the third is $\{1,3,5,4,1\}$).  The expected number of nodes visited by the first robot following $\rho_1$ is 3.88, and for two robots following $\rho_1$ and $\rho_2$ it is 4.905. Since there are only five nodes, it is clear that adding more robots must yield diminishing returns.

\begin{figure}[t] \label{fig:example}
	\centering
	\subfigure[Graph $\mathcal{G}$]{\label{fig:examplegraph}
		\includegraphics{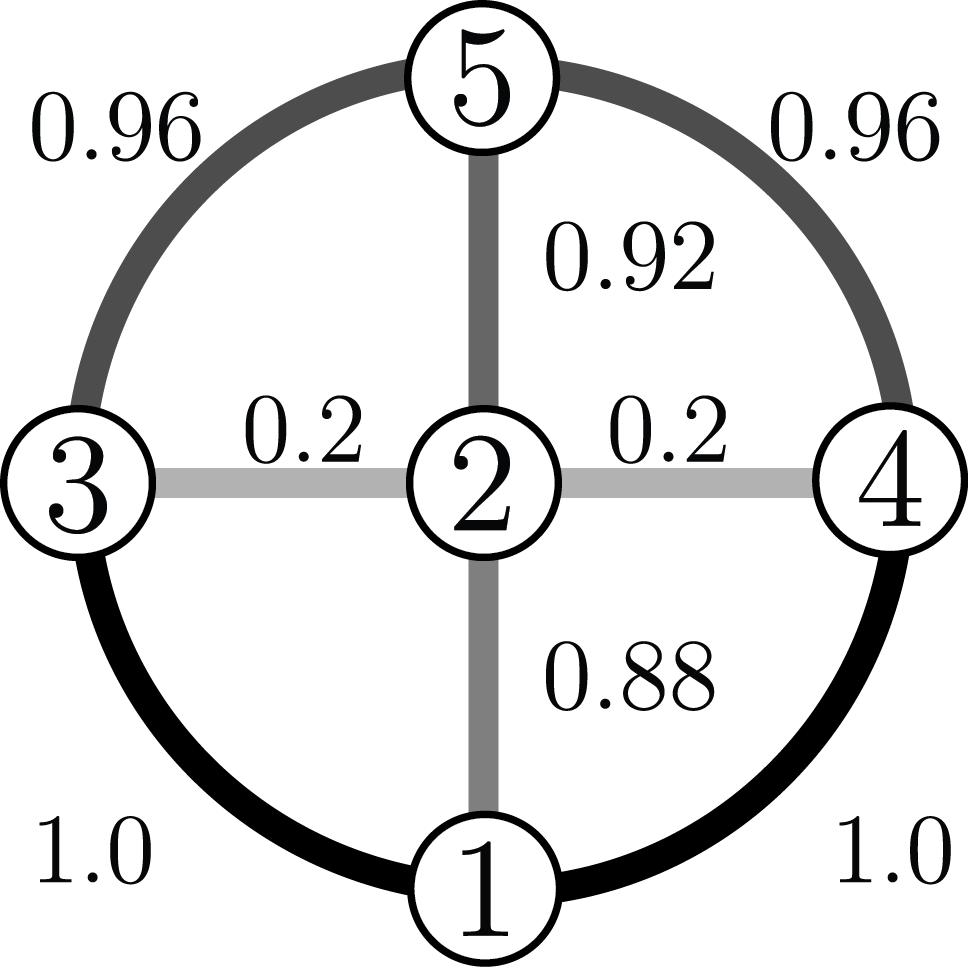}
        \hspace{.35em}
	}%
	\subfigure[Path for one robot]{\label{fig:exampleone}
    \centering
    \hspace{.6em}
		\includegraphics{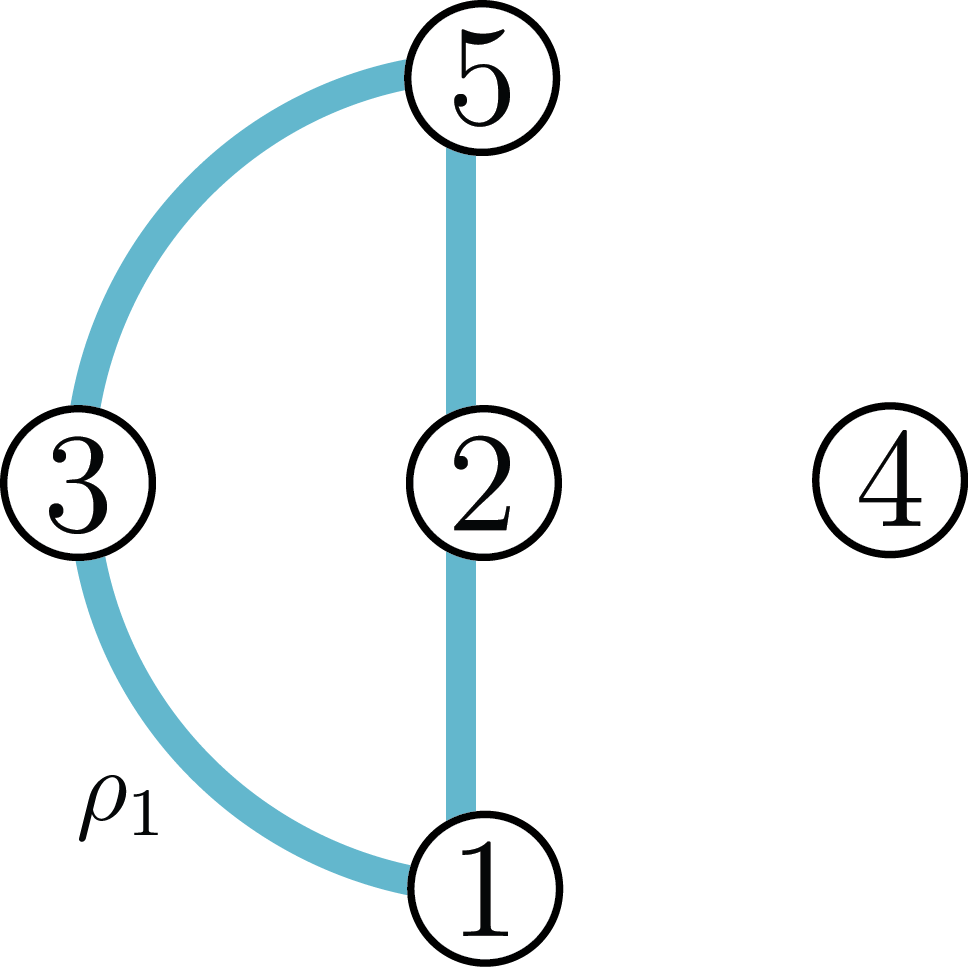}\hspace{.6em}
    }%
	\subfigure[Paths for two robots]{\label{fig:exampletwo}
    \centering
    \hspace{.35em}
		\includegraphics{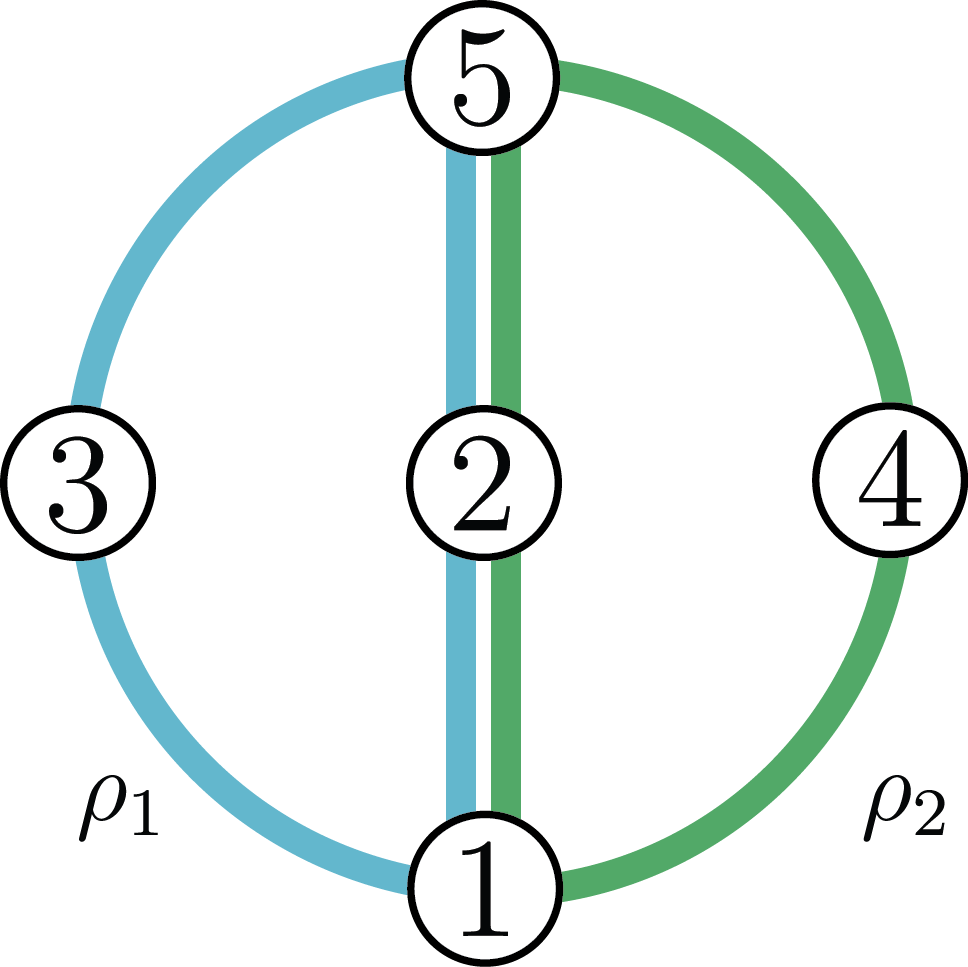}
        \hspace{.3em}
    }%
\caption{(a) Example of a TSO problem. Robots start at the bottom and darker lines correspond to safer paths. (b) A single robot can only visit four nodes safely. (c) Two robots can visit all nodes safely. It is easy to see that adding more robots yields diminishing returns.}
\end{figure}
\subsection{Variants and Applications} \label{sec:prob_variants}
\emph{Edge rewards and patrolling:} Our formulation can easily be extended to a scenario where the goal is to maximize the expected number of \emph{edges} visited by at least one robot. Define $z_{i,j}(\rho)$ to indicate whether a robot following path $\rho$ takes edge $(i,j)$, and for $(i,j)\in \mathcal{E}$ define $x_{i,j}(\{\rho_k\}_{k=1}^K)$ as before with $z_j$ replaced by $z_{i,j}$ (if $(i,j)\notin \mathcal{E}$, then define $x_{i,j}(\cdot) = 0$). The objective function for this problem is now:
$$\sum_{i=1}^V \sum_{j=1}^V d_{i,j}\mathbb{E}[x_{i,j}(\{\rho_{k}\}_{k=1}^K)].$$
This variant could be used to model a patrolling problem, where the goal is to inspect the maximum number of roads subject to the survival probability constraints. Such problems also occur when planning scientific missions (e.g., on Mars), where the objective is to execute the most important traversals.

\emph{Multiple visits and IPP:} We consider rewards for multiple visits as follows. Let $x^{(m)}_j$ indicate the event that node $j$ is visited by at least $m\ge 1$ robots, and let $d_j^{(m)}$ be the marginal benefit of the $m$th visit (for $m\le M$). Now the reward function is: 
$$\sum_{m=1}^M \sum_{j=1}^V d_j^{(m)}\mathbb{E}[x_j^{(m)}(\{\rho_k\}_{k=1}^K)].$$
 In order for our solution approach and guarantees to apply, we require that $d_j^{(m)}$ be a non-increasing function of $m$ (this ensures submodularity). We can build an approximation for any submodular function of the node visits by assigning $d_j^{(m)}$ to be the incremental gain for visiting node $j$ the $m$th time. A concrete example of this formulation is informative path planning where the goal is to maximize the reduction in entropy of the posterior distribution of node variables $\{f_j\}_{j=1}^V$, and $d_j^{(m)}$ represents the reduction in entropy of the posterior distribution of $f_j$ by taking the $m$th measurement. 



\section{Approximate solution approach}\label{sec:solution}
Our approach to solving the TSO problem is to exploit submodularity of the objective function and then derive a $\alpha$-approximate greedy algorithm (as defined in Section \ref{sec:BG:approx}). Accordingly, in Section \ref{sec:sol:submod} we show that the objective function of TSO is submodular. In Section \ref{sec:sol:linear} we present a linearization of the greedy sub-problem, which in the context of the TSO entails finding a path which maximizes the discrete derivative of the expected number of nodes visited, at the partial set already constructed. We use this linearization to find a polynomial time $(p_s/\lambda)$-approximate greedy algorithm. Leveraging this result, we describe our \texttt{GreedySurvivors} algorithm for the TSO problem in Section \ref{sec:sol:algo}, discuss its approximation guarantee in Section \ref{sec:sol:guarantee}, and characterize its computational complexity in Section \ref{sec:sol:complexity}. Finally, in Section \ref{sec:sol:variants} we discuss algorithm modifications for a number of variants of the TSO problem.

\subsection{Submodularity of the Objective Function}\label{sec:sol:submod}
In this section we show that the objective function is a normalized, non-negative monotone submodular function. Recall that submodularity can be checked by using the discrete derivative. For the TSO, a straightforward calculation gives the discrete derivative of the objective function as
\begin{equation*}
\objectiveGrad{\rho}{\Xk{K}} = \sum_{j=1}^V \mathbb{E}[z_j(\rho)] d_j\prod_{k=1}^K (1-\mathbb{E}[z_j(\rho_k)]).
\end{equation*}
The value placed on each node is the product of the probability that the robot visits the node, the importance of the node, and the probability the node has not been visited by any of the $K$ paths $\{\rho_k\}_{k=1}^K$.
\begin{lemma}[Objective is submodular]\label{lemma:submod} The objective function for the TSO, 
\begin{equation*}
J(\{\rho_k\}_{k=1}^K) = \objective{\Xk{K}},
\end{equation*}
is normalized, non-negative, monotone and submodular.
\end{lemma}

\begin{proof} The sum over an empty set is zero which immediately implies that the objective function is normalized. Because $d_j > 0$ and $\mathbb{E}[z_j(\cdot)] \in[0,1]$, the discrete derivative is everywhere non-negative. This implies that the objective function is both non-negative and monotone.
Now consider $X \subseteq X' \subseteq \feasible$ and $\rho \in \feasible \setminus X'$. To show submodularity, we must show that the discrete derivative is smaller at $X'$ than at $X$. Since $X \subseteq X'$ and $\mathbb{E}[z_j(\cdot)] \in [0,1]$,
\begin{equation*}
d_j\prod_{\rho_k\in X} (1-\mathbb{E}[z_j(\rho_k)]) \ge d_j\prod_{\rho_k\in X'} (1-\mathbb{E}[z_j(\rho_k)]).
\end{equation*}
This implies that
\begin{align*}
&\Delta J(x \mid X') = \sum_{j=1}^V \mathbb{E}[z_j(\rho)] d_j\prod_{\rho_k\in X'} (1-\mathbb{E}[z_j(\rho_k)])\\
&\qquad \le \sum_{j=1}^V \mathbb{E}[z_j(\rho)] d_j\prod_{\rho_k\in X} (1-\mathbb{E}[z_j(\rho_k)])=\Delta J(x\mid X).
\end{align*}
Therefore the objective function is submodular.
\end{proof}
Intuitively, this statement follows from the fact that the marginal gain of adding one more robot is proportional to the probability that nodes have not yet been visited, which is a decreasing function of the number of robots. This lemma shows that we may pose the TSO as a submodular maximization problem subject to a cardinality constraint, where $\mathcal{X}$ is the set of feasible paths and the cardinality constraint is the number of robots. 

\subsection{Linear Relaxation for Greedy Sub-problem}\label{sec:sol:linear}
As defined at the beginning of this section, the greedy sub-problem for the TSO at iteration $q$ requires us to find an element from $\feasible \setminus \{\rho_k\}_{k=1}^{q-1}$ which maximizes the discrete derivative at the partial set already constructed, $\{\rho_k\}_{k=1}^{q-1}$. This is very difficult for the TSO, because it requires finding a path which maximizes submodular node rewards subject to a distance constraint (this is the submodular orienteering problem). No polynomial time constant-factor approximation algorithm is known for general submodular orienteering problems \cite{CC-NK-MP:12}, and so in this section we design one specifically for the TSO.


We relax the problem by replacing the probability that the robot traversing path $\rho$ visits node $j$ by $\zeta_j$, which is the maximum probability that any robot following a feasible path can visit node $j$:
\begin{equation*}
\zeta_j = \max_{\rho \in \feasible} \mathbb{E}[z_j(\rho)].
\end{equation*}

For a given graph, this upper bound can be found easily by using Dijkstra's algorithm with log transformed edge weights $\logweight(e) := -\log(\edgeweight(e))$. Let $\mathbb{I}_j(\rho)$ be equal to 1 if $\rho$ attempts to visit node $j$ and 0 otherwise. For $q\le K$, let $c_j>0$ be the node weight $d_j$ times the probability that node $j$ has not been visited by robots following the paths $\{\rho_k\}_{k=1}^{q-1}$. We are then looking to find the path that maximizes the sum:

\begin{equation*}
\Delta\bar{J}(\rho \mid \{\rho_k\}_{k=1}^{q-1}) := \sum_{j=1}^V \mathbb{I}_j(\rho)\zeta_j c_j,
\end{equation*}
which represents an {\em optimistic} estimate of the actual reward. We can find this path by solving an orienteering problem: Recall that for the orienteering problem we provide node weights and a constraint on the sum of edge weights (referred to as a budget), and find the path which maximizes the node rewards along the path while guaranteeing that the sum of edge weights along the path is below the budget.

We use the modified graph $\mathcal{G}_{O}$, which has the same edges and nodes as $\mathcal{G}$ but has  edge weights $\logweight(e)$, budget $-\log(p_s)$, and node rewards $\nodeweight_q(j) = \zeta_jc_j$. Solving the orienteering problem on $\mathcal{G}_{O}$ will return a path such that $\sum_{e \in \rho} -\log(\edgeweight(e)) \le -\log(p_s)$, which is equivalent to $\mathbb{P}\{a_{\lvert\rho\rvert}(\rho)=1\} \ge p_s$, and the path will maximize the sum of node rewards, which is $\Delta\bar{J}(\rho \mid \{\rho_k\}_{k=1}^{q-1}$. 
  
Although solving the orienteering problem is NP-hard, several polynomial-time constant-factor approximation algorithms exist which guarantee that the returned objective is lower bounded by a factor of $1/\lambda \le 1$ of the optimal objective. For undirected planar graphs, \cite{KC-HPS:06} gives a guarantee $\lambda = (1+\epsilon)$, for undirected graphs \cite{CC-NK-MP:12} gives a guarantee $\lambda = (2+\epsilon)$, and for directed graphs \cite{CC-MP:05} gives a guarantee in terms of the number of nodes. Using such an oracle, we have the following guarantee: 
\begin{lemma}[Single robot constant-factor guarantee]\label{lemma:approx_greedy} Let \texttt{Orienteering} be a routine that solves the orienteering problem within constant-factor $1/\lambda$, that is for node weights $\nodeweight(j)$, path $\hat{\rho}$ output by the routine and any path $\rho \in \feasible$, 
\begin{equation*}
 \sum_{j=1}^V \mathbb{I}_j(\hat{\rho})\nodeweight(j) \ge  \frac{1}{\lambda}\sum_{j=1}^V \mathbb{I}_j(\rho)\nodeweight(j). 
\end{equation*}

Then for any $c_j > 0$ and any  $\rho \in \feasible$, the weighted expected number of nodes visited by a robot following path $\hat{\rho}$ satisfies
\begin{equation*}
\sum_{j=1}^Vc_j\mathbb{E}[z_j(\hat{\rho})] \ge \frac{p_s}{\lambda}\sum_{j=1}^V c_j\mathbb{E}[z_j(\rho)].
\end{equation*}
\end{lemma}

\begin{proof}
By definition of $\zeta_j$ and the \texttt{Orienteering} routine, we have:
\begin{align*}
\sum_{j=1}^V c_j\mathbb{E}[z_j(\rho)] &\le \sum_{j=1}^V \mathbb{I}_j(\rho)\zeta_jc_j \le \lambda \sum_{j=1}^V \mathbb{I}_j(\hat{\rho})\zeta_jc_j.
\end{align*}
Because path $\hat{\rho}$ is feasible $\mathbb{I}_j(\hat{\rho})p_s\zeta_j \le \mathbb{I}_j(\hat{\rho})p_s  \le \mathbb{E}[z_j(\hat{\rho})] $, which combined with the equation above completes the proof.

\end{proof}

This is a remarkable statement because it guarantees that, as long as $p_s$ is not too small, the solution to the linear relaxation will give nearly the optimal value in the original problem. The intuition is that for $p_s$ close to unity no feasible path can be very risky and so the probability that a robot \emph{actually} reaches a node will not be too far from the maximum probability that it \emph{could} reach the node.

\subsection{Greedy Approximation for the TSO}\label{sec:sol:algo}
By choosing $c_j$ to be the node weight $d_j$ times the probability that node $j$ has \emph{not} yet been visited, the linearized greedy algorithm above has guarantee $\alpha = p_s/\lambda$, which means we can use it as an $(p_s/\lambda)$-approximate greedy selection step to construct a set of $K$ paths.

To get the upper bounds $\zeta_j$ we define the method \texttt{Dijkstra}($\mathcal{G}$, $i$, $j$), which returns the length of the shortest path from $i$ to $j$ on the edge weighted graph $\mathcal{G}$ using Dijkstra's algorithm. The \texttt{Orienteering}$(\mathcal{G}, \nodeweight)$ routine solves the orienteering problem (assuming $v_s = 1$, $v_t = V$) within factor $1/\lambda$ given an edge weighted graph $\mathcal{G}$ and node rewards $\nodeweight$. Pseudocode for our algorithm is given in Figure \ref{alg:heur}. We begin by forming the graph $\mathcal{G}_{O}$ with log-transformed edge weights $\logweight(e)$, and then use Dijkstra's algorithm to compute the maximum probability that a node can be reached. For each robot $k=1,\dots,K$, we solve the orienteering problem to greedily choose paths that maximize the discrete derivative of $\bar{J}$, updating the derivative after choosing each path. \\

\begin{figure}[h]
\begin{algorithmic}[1]
\Procedure{GreedySurvivors}{$\mathcal{G},K$}
\State Form $\mathcal{G}_{O}$ from $\mathcal{G}$, such that $v_s=1$, $v_t = V$
\For{$j=1,\dots,V$}
\State $\zeta_j \gets \exp(-\texttt{Dijkstra}(\mathcal{G}_{O},1,j))$
\State $\nodeweight_1(j) \gets \zeta_jd_j$
\EndFor
\State $\rho_1 \gets \texttt{Orienteering}(\mathcal{G}_{O},\nodeweight_1)$ 
\For{$k=1,\dots,K-1$}
\State $\mathbb{E}[a_0(\rho_k)] \gets 1$
\For{$n = 1,\dots,\lvert \rho_k \rvert$}
\State $\mathbb{E}[a_n(\rho_k)] \gets \mathbb{E}[a_{n-1}(\rho_k)]\edgeweight(e_{\rho_k}^n)$ 
\State $\nodeweight_{k+1}(\rho_k(n)) \gets (1-\mathbb{E}[a_n(\rho_k)])\nodeweight_k(\rho_k(n))$
\EndFor
\State $\rho_{k+1} \gets \texttt{Orienteering}(\mathcal{G}_{O}, \nodeweight_{k+1})$
\EndFor
\EndProcedure
\end{algorithmic}˛
\caption{Approximate greedy algorithm for solving the TSO problem. }
\label{alg:heur}
\end{figure}

\subsection{Approximation Guarantees}\label{sec:sol:guarantee}
In this section we combine the results from Section \ref{sec:BG:approx} and \ref{sec:sol:linear} to give a constant-factor approximation for the \texttt{GreedySurvivors} algorithm:

\begin{theorem}[Multi-robot constant-factor guarantee]\label{thm:approx_guarantee} Let $1/\lambda$ be the constant-factor guarantee for the \texttt{Orienteering} routine as in Lemma 1, and assign robot $\ell$ the path $\hat{\rho}_{\ell}$ output by the orienteering routine given graph $\mathcal{G}_\text{O}$ with node weights

\begin{equation*} 
\nodeweight_\ell(j) = \zeta_jd_j\prod_{\iota=1}^{\ell-1}\left(1-\mathbb{E}\left[z_j(\hat{\rho}_{\iota})\right]\right).
\end{equation*}

Let $X^*_K = \{\rho_k^*\}_{k=1}^K$ be an optimal solution to the TSO with $K$ robots. Then the weighted expected number of nodes visited by a team of $L\ge K$ robots following the paths $\hat{X}_L = \{\hat{\rho}_{\ell}\}_{\ell=1}^L$ is at least
\begin{align*}
&\objective{\hat{X}_L} \ge \left(1-e^{-\frac{p_sL}{\lambda K}}\right)\objective{X^*_K}.
\end{align*}
\end{theorem}
\begin{proof}
Using Lemma \ref{lemma:approx_greedy} with $c_j = d_j\prod_{\iota=1}^{\ell-1}(1-\mathbb{E}[z_j(\hat{\rho}_{\iota})])$, we have a constant-factor guarantee $\alpha = p_s/\lambda$ for the linearized greedy algorithm. Applying Theorem \ref{thm:extension} to our objective function (which by Lemma \ref{lemma:submod} is normalized non-negative, monotone, and submodular) we have the desired result.
\end{proof}

In many scenarios of interest $p_s$ is quite close to 1, since robots are quite valuable.  For $L=K$ this theorem gives an $\approxfact$ guarantee for the output of our algorithm.  This bound holds for any team size, and guarantees that the output of the  (polynomial time) linearized greedy algorithm will have a similar reward to the output of the (exponential time) optimal algorithm.

Taking $L > K$ gives a practical way of testing how much more efficient the allocation for $K$ robots could be. For example, if $L\frac{p_s}{\lambda}=6K$ we have a $(1-1/e^6) \simeq 0.997$ factor approximation for the optimal value achieved by $K$ robots. We use this approach to generate tight upper bounds for our experimental results.

Note that this theorem also guarantees that as $L\to\infty$, the output of our algorithm has at least the same value as the optimum, which emphasizes the importance of guarantees for \emph{small} teams. 

\subsection{Computational Complexity}\label{sec:sol:complexity}
Suppose that the complexity of the $\texttt{Orienteering}$ oracle is $C_O$. Then the complexity of our algorithm is:

\begin{equation*}
O(V^2\log(V)) + O(KV^2) + O(KC_O) = O(KC_O).
\end{equation*}

The first term is the complexity of running Dijkstra's to calculate $\zeta_j$ for all nodes, the second term is the complexity of updating the $V$ weights $K$ times (each update costs at most $\lvert \rho_k\rvert \le V$ flops), and the final term is the complexity of solving the $K$ orienteering problems. For many approximation algorithms $C_O = V^{O(1/\epsilon)}$, and so the complexity is dominated by $KC_O$. If a suitable approximation algorithm is used for $\texttt{Orienteering}$ (such as \cite{CC-MP:05}, \cite{CC-NK-MP:12}, \cite{KC-HPS:06}), this algorithm will have reasonable computation time even for large team sizes.

\subsection{Algorithm Variants}\label{sec:sol:variants}
Below we describe how to solve the variants from Section \ref{sec:prob_variants} by modifying the \texttt{GreedySurvivors} routine.
\subsubsection{Edge Rewards and Patrolling} After redefining the problem variables as described in Section \ref{sec:prob_variants} we can define $\zeta_{i,j} = \zeta_i \edgeweight(i,j)$, which is the largest probability that edge $(i,j)$ is successfully taken. The linearized greedy algorithm will still have constant-factor guarantee $\alpha = p_s/\lambda$, but now requires solving an \emph{arc} orienteering problem. Constant-factor approximations for the arc orienteering problem can be found using algorithms for the OP as demonstrated in \cite{DG-CK-KM-GP-NV:15}: for an undirected graph $\lambda = 6+\epsilon + o(1)$ in polynomial time $V^{O(1/\epsilon)}$. The arguments for Theorem \ref{thm:approx_guarantee} are the same as in the node reward case.

\subsubsection{Multiple Visits and IPP}
The multiple visits variant adds rewards for visiting a node up to $M$ times. We can linearize the problem by choosing $c_j$ as the sum over $m=1,\dots,M$ of $d_j^{(m)}$ times the probability that exactly $m-1$ previous robots visit node $j$.  Because $c_j$ is still a positive constant, we can apply Lemma \ref{lemma:approx_greedy}.  The only step left to show is that the objective is still submodular, for which we require that $d_j^{(m)}$ be a non-increasing function in $m$. To see why this is the case, consider two teams which each visit node $j$ once. The cumulative reward the teams receive for visiting node $j$ is $2d_j^{(1)}$. If the teams are combined, then node $j$ is visited twice and so the combined team gets reward $d_j^{(1)}+d_j^{(2)} \le 2d_j^{(1)}$.

It is important to note that the complexity results change unfavorably. To linearize the greedy problem, we must compute the probability that exactly $m \le K$ robots visits node $j$, which requires evaluating the $K$ choose $m$ visit events. If the number of profitable visits is at most $M<K$ then the number of visit events is a polynomial function of the team size (bounded by $K^M/M!$), but if $K < M$ then there are $2^K$ visit events which must be evaluated.

\section{Numerical Experiments}\label{sec:numbers}
\subsection{Verification of Bounds}
We consider a TSO problem on the graph shown in Figure \ref{fig:hexgraph}: the central starting node has `safe' transitions to six nodes, which have `unsafe' transitions to the remaining twelve nodes. Due to the symmetry of the problem we can quickly compute an optimal policy for a team of six robots, which is shown in Figure \ref{fig:hexoptplan}. The output of the greedy algorithm is shown in Figure \ref{fig:hexplan}. The \texttt{GreedySurvivors} solution comes close to the optimal, although the initial path planned (shown by the thick dark blue line) does not anticipate its impact on later paths. The expected number of nodes visited by robots following optimal paths, greedy paths, and the upper bound are shown in Figure \ref{fig:performance}. Note that the upper bound is close to the optimal, even for small teams, and that the \texttt{GreedySurvivors} performance is nearly optimal.

\input{Figures/hex.tex}

\begin{figure}[h]
    \centering
    \includegraphics{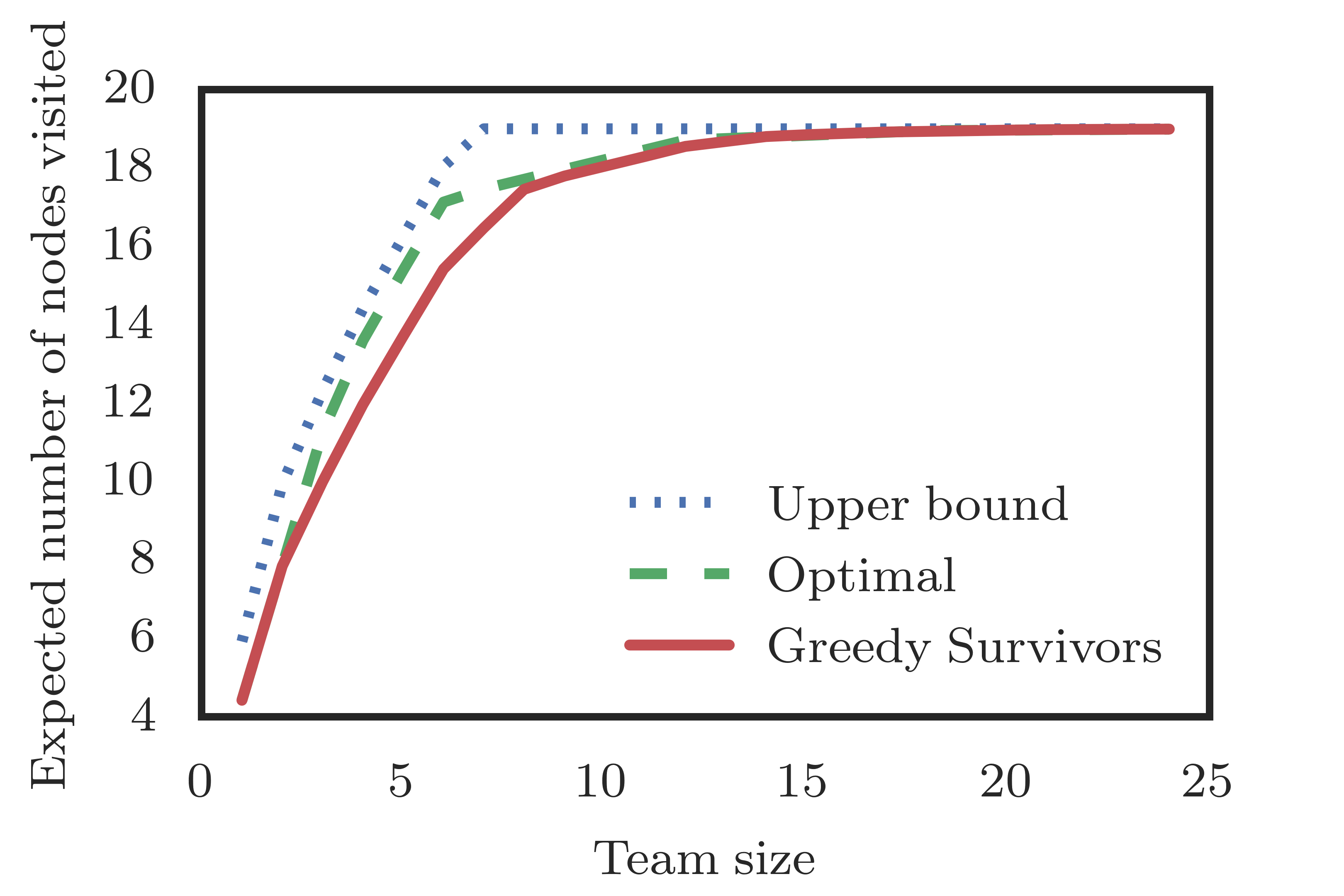}
    \caption{Performance comparison for the example in Figure \ref{fig:hexgraph}. The optimal value is shown in green and the GreedySurvivors value is shown in red. The upper bound on the optimum from Theorem \ref{thm:approx_guarantee} is shown by the dotted line.}
    \label{fig:performance}
\end{figure}

\subsection{Empirical Approximation Factor}
We compare our algorithm's performance against an upper bound on the optimal value. We use an exact solver for the orienteering problem (using the Gurobi MIP solver), and generate instances on a graph with $V=65$ nodes and uniformly distributed edge weights in the interval $[0.3,1)$. The upper bound used for comparison is the smallest of 1) the number of nodes which can be reached within the budget, 2) the constant-factor guarantee times our approximate solution, and 3) the guarantee from solving the problem with an oversized team (from Theorem \ref{thm:approx_guarantee}). The average performance (relative to the upper bound) along with the total range of results are shown in Figure \ref{fig:bounds}, with the function $\approxfact$ drawn as a dashed line. As shown, the approximation factor converges to the optimal as the team size grows. The dip around $p_s = 0.85$ is due to looseness in the bound and the fact that the optimum is not yet reached by the greedy routine.
\begin{figure}
    \centering
    \includegraphics{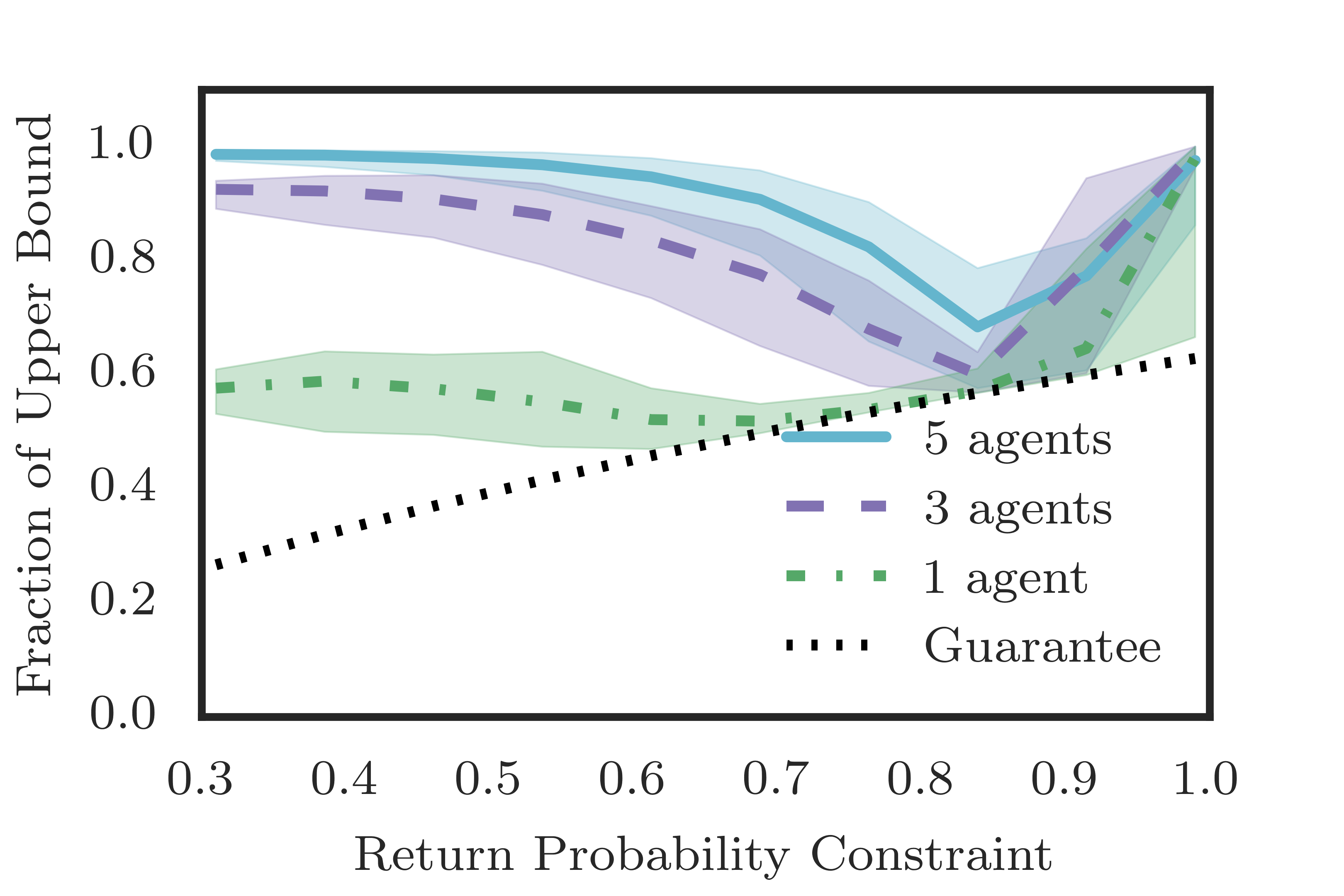}
    \caption{Ratio of actual result to upper bound for a 65 node complete graph. The team size ranges from 1 (at the bottom) to 5 (at the top), and in all cases a significant fraction of the possible reward is accumulated even for small $p_s$.}
    \label{fig:bounds}
\end{figure}

\subsection{Large Scale Performance}
We demonstrate the run-time of \texttt{GreedySurvivors} for large-scale problems by planning $K=25$ paths for complete graphs of various sizes.  We use two \texttt{Orienteering} routines: the mixed integer formulation from \cite{IK-PSB-TD:16} with Gurobi's MIP solver, and an adapted version of the open source heuristic developed by the authors of \cite{PV-WS-GVB-DVO:09}. We use a heuristic approach because in practice it performs better than a polynomial time approximation algorithm. For the cases where we have comparison data (up to $V = 100$ nodes) the heuristic achieves an average of 0.982 the reward of the MIP algorithm. Even very large problems, e.g. 25 robots on a 900 node graph, can be solved in approximately an hour with the heuristic on a machine that has a 3GHz i7 processor using 8 cores and 64GB of RAM.

\section{Conclusion}\label{sec:conclusion}
In this paper we formulate the \emph{Team Surviving Orienteers problem}, where we are asked to maximize the expected number of nodes visited while guaranteeing that every robot survives with probability at least $p_s$. What sets this problem apart from previous work is the notion of risky traversal, where a robot might not complete its planned path. This introduces a difficult combination of submodular objective and survival probability constraints. We develop the \texttt{GreedySurvivors} algorithm which has polynomial time complexity with a constant-factor guarantee that the returned objective is $\Omega((\approxfact)\text{OPT})$, where $\text{OPT}$ is the optimum. We demonstrate the effectiveness of our algorithm in numerical simulations and discuss extensions to several variants of the TSO problem.

There are numerous directions for future work: First, an on-line version of this algorithm would react to knowledge of robot failure and re-plan the paths without exposing the surviving robots to more risk. Second, considering non-homogeneous teams would expand the many practical applications of the TSO problem. Third, extending the analysis to {\em walks} on a graph (where a robot can re-visit nodes) would allow for a broader set of solutions and may yield better performance.  Finally, we are interested in using some of the concepts from \cite{AMC-MG-BWT:11} to consider more general probability models for the TSO.

\section*{Acknowledgements}
The authors would like to thank Federico Rossi and Edward Schmerling for their insights which led to tighter analysis. 
\bibliographystyle{IEEEtran}
{\renewcommand{\baselinestretch}{.9}
\bibliography{main}}

\end{document}

%% file: Figures/hex.tex
\begin{figure}[t]
\label{fig:hex}
	\centering
	\subfigure[Graph $\mathcal{G}$]{\label{fig:hexgraph}
		\centering
		\begin{tikzpicture}
		\definecolor{c1}{rgb}{0.2980392156862745,0.4470588235294118,0.6901960784313725}
		\definecolor{c2}{rgb}{0.3333333333333333,0.6588235294117647,0.40784313725490196}
		\definecolor{c3}{rgb}{0.7686274509803922,0.3058823529411765,0.3215686274509804}
		\definecolor{c4}{rgb}{0.5058823529411764,0.4470588235294118,0.6980392156862745}
		\definecolor{c5}{rgb}{0.8,0.7254901960784313,0.4549019607843137}
		\definecolor{c6}{rgb}{0.39215686274509803,0.7098039215686275,0.803921568627451}
		\definecolor{bg}{rgb}{0.3,0.3,0.3}
		\node () at (0,0){};
		\draw[opacity=0.4] (1.0,1.0) edge[-,line width=1pt, color=bg] (1.0,0.5);
		\draw[opacity=0.4] (1.0,1.0) edge[-,line width=1pt, color=bg] (1.4330127018922192,0.75);
		\draw[opacity=0.4] (1.0,1.0) edge[-,line width=1pt, color=bg] (0.5669872981077807,0.75);
		\draw (1.0,0.5) edge[-,line width=2.0pt, color=bg] (1.0,0.0);
		\draw[opacity=0.4] (1.0,0.5) edge[-,line width=1pt, color=bg] (1.4330127018922192,0.75);
		\draw[opacity=0.4] (1.0,0.5) edge[-,line width=1pt, color=bg] (1.4330127018922194,0.25);
		\draw[opacity=0.4] (1.0,0.5) edge[-,line width=1pt, color=bg] (0.5669872981077807,0.75);
		\draw[opacity=0.4] (1.0,0.5) edge[-,line width=1pt, color=bg] (0.5669872981077808,0.25);
		\draw (1.0,0.0) edge[-,line width=2.0pt, color=bg] (1.0,-0.5);
		\draw (1.0,0.0) edge[-,line width=2.0pt, color=bg] (1.4330127018922194,0.25);
		\draw (1.0,0.0) edge[-,line width=2.0pt, color=bg] (1.4330127018922194,-0.25);
		\draw (1.0,0.0) edge[-,line width=2.0pt, color=bg] (0.5669872981077808,0.25);
		\draw (1.0,0.0) edge[-,line width=2.0pt, color=bg] (0.5669872981077808,-0.25);
		\draw[opacity=0.4] (1.0,-0.5) edge[-,line width=1pt, color=bg] (1.0000000000000002,-1.0);
		\draw[opacity=0.4] (1.0,-0.5) edge[-,line width=1pt, color=bg] (1.4330127018922194,-0.25);
		\draw[opacity=0.4] (1.0,-0.5) edge[-,line width=1pt, color=bg] (1.4330127018922194,-0.75);
		\draw[opacity=0.4] (1.0,-0.5) edge[-,line width=1pt, color=bg] (0.5669872981077808,-0.25);
		\draw[opacity=0.4] (1.0,-0.5) edge[-,line width=1pt, color=bg] (0.5669872981077808,-0.75);
		\draw[opacity=0.4] (1.0000000000000002,-1.0) edge[-,line width=1pt, color=bg] (1.4330127018922194,-0.75);
		\draw[opacity=0.4] (1.0000000000000002,-1.0) edge[-,line width=1pt, color=bg] (0.5669872981077808,-0.75);
		\draw[opacity=0.4] (1.4330127018922192,0.75) edge[-,line width=1pt, color=bg] (1.4330127018922194,0.25);
		\draw[opacity=0.4] (1.4330127018922192,0.75) edge[-,line width=1pt, color=bg] (1.8660254037844386,0.5);
		\draw[opacity=0.4] (1.4330127018922194,0.25) edge[-,line width=1pt, color=bg] (1.4330127018922194,-0.25);
		\draw[opacity=0.4] (1.4330127018922194,0.25) edge[-,line width=1pt, color=bg] (1.8660254037844386,0.5);
		\draw[opacity=0.4] (1.4330127018922194,0.25) edge[-,line width=1pt, color=bg] (1.8660254037844388,0.0);
		\draw[opacity=0.4] (1.4330127018922194,-0.25) edge[-,line width=1pt, color=bg] (1.4330127018922194,-0.75);
		\draw[opacity=0.4] (1.4330127018922194,-0.25) edge[-,line width=1pt, color=bg] (1.8660254037844388,0.0);
		\draw[opacity=0.4] (1.4330127018922194,-0.25) edge[-,line width=1pt, color=bg] (1.8660254037844388,-0.5);
		\draw[opacity=0.4] (1.4330127018922194,-0.75) edge[-,line width=1pt, color=bg] (1.8660254037844388,-0.5);
		\draw[opacity=0.4] (0.5669872981077807,0.75) edge[-,line width=1pt, color=bg] (0.5669872981077808,0.25);
		\draw[opacity=0.4] (0.5669872981077807,0.75) edge[-,line width=1pt, color=bg] (0.1339745962155614,0.5);
		\draw[opacity=0.4] (0.5669872981077808,0.25) edge[-,line width=1pt, color=bg] (0.5669872981077808,-0.25);
		\draw[opacity=0.4] (0.5669872981077808,0.25) edge[-,line width=1pt, color=bg] (0.1339745962155614,0.5);
		\draw[opacity=0.4] (0.5669872981077808,0.25) edge[-,line width=1pt, color=bg] (0.13397459621556151,0.0);
		\draw[opacity=0.4] (0.5669872981077808,-0.25) edge[-,line width=1pt, color=bg] (0.5669872981077808,-0.75);
		\draw[opacity=0.4] (0.5669872981077808,-0.25) edge[-,line width=1pt, color=bg] (0.13397459621556151,0.0);
		\draw[opacity=0.4] (0.5669872981077808,-0.25) edge[-,line width=1pt, color=bg] (0.13397459621556151,-0.5);
		\draw[opacity=0.4] (0.5669872981077808,-0.75) edge[-,line width=1pt, color=bg] (0.13397459621556151,-0.5);
		\draw[opacity=0.4] (1.8660254037844386,0.5) edge[-,line width=1pt, color=bg] (1.8660254037844388,0.0);
		\draw[opacity=0.4] (1.8660254037844388,0.0) edge[-,line width=1pt, color=bg] (1.8660254037844388,-0.5);
		\draw[opacity=0.4] (0.1339745962155614,0.5) edge[-,line width=1pt, color=bg] (0.13397459621556151,0.0);
		\draw[opacity=0.4] (0.13397459621556151,0.0) edge[-,line width=1pt, color=bg] (0.13397459621556151,-0.5);
		\node[circle,fill=black,minimum size = 5pt, inner sep=0pt] () at (1.0,1.0) {};
		\node[circle,fill=black,minimum size = 5pt, inner sep=0pt] () at (1.0,0.5) {};
		\node[circle,fill=black,minimum size = 5pt, inner sep=0pt] () at (1.0,0.0) {};
		\node[circle,fill=black,minimum size = 5pt, inner sep=0pt] () at (1.0,-0.5) {};
		\node[circle,fill=black,minimum size = 5pt, inner sep=0pt] () at (1.0000000000000002,-1.0) {};
		\node[circle,fill=black,minimum size = 5pt, inner sep=0pt] () at (1.4330127018922192,0.75) {};
		\node[circle,fill=black,minimum size = 5pt, inner sep=0pt] () at (1.4330127018922194,0.25) {};
		\node[circle,fill=black,minimum size = 5pt, inner sep=0pt] () at (1.4330127018922194,-0.25) {};
		\node[circle,fill=black,minimum size = 5pt, inner sep=0pt] () at (1.4330127018922194,-0.75) {};
		\node[circle,fill=black,minimum size = 5pt, inner sep=0pt] () at (0.5669872981077807,0.75) {};
		\node[circle,fill=black,minimum size = 5pt, inner sep=0pt] () at (0.5669872981077808,0.25) {};
		\node[circle,fill=black,minimum size = 5pt, inner sep=0pt] () at (0.5669872981077808,-0.25) {};
		\node[circle,fill=black,minimum size = 5pt, inner sep=0pt] () at (0.5669872981077808,-0.75) {};
		\node[circle,fill=black,minimum size = 5pt, inner sep=0pt] () at (1.8660254037844386,0.5) {};
		\node[circle,fill=black,minimum size = 5pt, inner sep=0pt] () at (1.8660254037844388,0.0) {};
		\node[circle,fill=black,minimum size = 5pt, inner sep=0pt] () at (1.8660254037844388,-0.5) {};
		\node[circle,fill=black,minimum size = 5pt, inner sep=0pt] () at (0.1339745962155614,0.5) {};
		\node[circle,fill=black,minimum size = 5pt, inner sep=0pt] () at (0.13397459621556151,0.0) {};
		\node[circle,fill=black,minimum size = 5pt, inner sep=0pt] () at (0.13397459621556151,-0.5) {};
		\end{tikzpicture}}%
		\hspace{1.5em}
		\subfigure[Optimal $X_6^*$]{\label{fig:hexoptplan}
		\centering
		\begin{tikzpicture}
		\definecolor{c1}{rgb}{0.2980392156862745,0.4470588235294118,0.6901960784313725}
		\definecolor{c2}{rgb}{0.3333333333333333,0.6588235294117647,0.40784313725490196}
		\definecolor{c3}{rgb}{0.7686274509803922,0.3058823529411765,0.3215686274509804}
		\definecolor{c4}{rgb}{0.5058823529411764,0.4470588235294118,0.6980392156862745}
		\definecolor{c5}{rgb}{0.8,0.7254901960784313,0.4549019607843137}
		\definecolor{c6}{rgb}{0.39215686274509803,0.7098039215686275,0.803921568627451}
		\definecolor{bg}{rgb}{0.3,0.3,0.3}
		\draw (1.0,0.0) edge[-,line width=2.5pt, color=c1] (1.0,0.5);
		\draw (1.0,0.5) edge[-,line width=2.5pt, color=c1] (1.0,1.0);
		\draw (1.0,1.0) edge[-,line width=2.5pt, color=c1] (1.4330127018922192,0.75);
		\draw (1.4330127018922192,0.75) edge[-,line width=2.5pt, color=c1] (1.4330127018922194,0.25);
		\draw (1.4330127018922194,0.25) edge[-,line width=2.5pt, color=c1] (1.0,0.0);
		\draw (1.0,0.0) edge[-,line width=1.5pt, color=c2] (1.0,0.5);
		\draw (1.0,0.5) edge[-,line width=1.5pt, color=c2] (0.5669872981077807,0.75);
		\draw (0.5669872981077807,0.75) edge[-,line width=1.5pt, color=c2] (0.1339745962155614,0.5);
		\draw (0.1339745962155614,0.5) edge[-,line width=1.5pt, color=c2] (0.5669872981077808,0.25);
		\draw (0.5669872981077808,0.25) edge[-,line width=1.5pt, color=c2] (1.0,0.0);
		\draw (1.0,0.0) edge[-,line width=1.5pt, color=c3] (1.4330127018922194,0.25);
		\draw (1.4330127018922194,0.25) edge[-,line width=1.5pt, color=c3] (1.8660254037844386,0.5);
		\draw (1.8660254037844386,0.5) edge[-,line width=1.5pt, color=c3] (1.8660254037844388,0.0);
		\draw (1.8660254037844388,0.0) edge[-,line width=1.5pt, color=c3] (1.4330127018922194,-0.25);
		\draw (1.4330127018922194,-0.25) edge[-,line width=1.5pt, color=c3] (1.0,0.0);
		\draw (1.0,0.0) edge[-,line width=1.5pt, color=c4] (1.4330127018922194,-0.25);
		\draw (1.4330127018922194,-0.25) edge[-,line width=1.5pt, color=c4] (1.8660254037844388,-0.5);
		\draw (1.8660254037844388,-0.5) edge[-,line width=1.5pt, color=c4] (1.4330127018922194,-0.75);
		\draw (1.4330127018922194,-0.75) edge[-,line width=1.5pt, color=c4] (1.0,-0.5);
		\draw (1.0,-0.5) edge[-,line width=1.5pt, color=c4] (1.0,0.0);
		\draw (1.0,0.0) edge[-,line width=1.5pt, color=c5] (1.0,-0.5);
		\draw (1.0,-0.5) edge[-,line width=1.5pt, color=c5] (1.0000000000000002,-1.0);
		\draw (1.0000000000000002,-1.0) edge[-,line width=1.5pt, color=c5] (0.5669872981077808,-0.75);
		\draw (0.5669872981077808,-0.75) edge[-,line width=1.5pt, color=c5] (0.5669872981077808,-0.25);
		\draw (0.5669872981077808,-0.25) edge[-,line width=1.5pt, color=c5] (1.0,0.0);
		\draw (1.0,0.0) edge[-,line width=1.5pt, color=c6] (0.5669872981077808,-0.25);
		\draw (0.5669872981077808,-0.25) edge[-,line width=1.5pt, color=c6] (0.13397459621556151,-0.5);
		\draw (0.13397459621556151,-0.5) edge[-,line width=1.5pt, color=c6] (0.13397459621556151,0.0);
		\draw (0.13397459621556151,0.0) edge[-,line width=1.5pt, color=c6] (0.5669872981077808,0.25);
		\draw (0.5669872981077808,0.25) edge[-,line width=1.5pt, color=c6] (1.0,0.0);
		\node[circle,fill=black,minimum size = 5pt, inner sep=0pt] () at (1.0,1.0) {};
		\node[circle,fill=black,minimum size = 5pt, inner sep=0pt] () at (1.0,0.5) {};
		\node[circle,fill=black,minimum size = 5pt, inner sep=0pt] () at (1.0,0.0) {};
		\node[circle,fill=black,minimum size = 5pt, inner sep=0pt] () at (1.0,-0.5) {};
		\node[circle,fill=black,minimum size = 5pt, inner sep=0pt] () at (1.0000000000000002,-1.0) {};
		\node[circle,fill=black,minimum size = 5pt, inner sep=0pt] () at (1.4330127018922192,0.75) {};
		\node[circle,fill=black,minimum size = 5pt, inner sep=0pt] () at (1.4330127018922194,0.25) {};
		\node[circle,fill=black,minimum size = 5pt, inner sep=0pt] () at (1.4330127018922194,-0.25) {};
		\node[circle,fill=black,minimum size = 5pt, inner sep=0pt] () at (1.4330127018922194,-0.75) {};
		\node[circle,fill=black,minimum size = 5pt, inner sep=0pt] () at (0.5669872981077807,0.75) {};
		\node[circle,fill=black,minimum size = 5pt, inner sep=0pt] () at (0.5669872981077808,0.25) {};
		\node[circle,fill=black,minimum size = 5pt, inner sep=0pt] () at (0.5669872981077808,-0.25) {};
		\node[circle,fill=black,minimum size = 5pt, inner sep=0pt] () at (0.5669872981077808,-0.75) {};
		\node[circle,fill=black,minimum size = 5pt, inner sep=0pt] () at (1.8660254037844386,0.5) {};
		\node[circle,fill=black,minimum size = 5pt, inner sep=0pt] () at (1.8660254037844388,0.0) {};
		\node[circle,fill=black,minimum size = 5pt, inner sep=0pt] () at (1.8660254037844388,-0.5) {};
		\node[circle,fill=black,minimum size = 5pt, inner sep=0pt] () at (0.1339745962155614,0.5) {};
		\node[circle,fill=black,minimum size = 5pt, inner sep=0pt] () at (0.13397459621556151,0.0) {};
		\node[circle,fill=black,minimum size = 5pt, inner sep=0pt] () at (0.13397459621556151,-0.5) {};
		\end{tikzpicture}}%
		\hspace{1.5em}
		\subfigure[Greedy $\bar{X}_6$]{\label{fig:hexplan}
		\centering
		\begin{tikzpicture}
		\definecolor{c1}{rgb}{0.2980392156862745,0.4470588235294118,0.6901960784313725}
		\definecolor{c2}{rgb}{0.3333333333333333,0.6588235294117647,0.40784313725490196}
		\definecolor{c3}{rgb}{0.7686274509803922,0.3058823529411765,0.3215686274509804}
		\definecolor{c4}{rgb}{0.5058823529411764,0.4470588235294118,0.6980392156862745}
		\definecolor{c5}{rgb}{0.8,0.7254901960784313,0.4549019607843137}
		\definecolor{c6}{rgb}{0.39215686274509803,0.7098039215686275,0.803921568627451}
		\definecolor{bg}{rgb}{0.3,0.3,0.3}
		\draw (1.0,0.0) edge[-,line width=2.5pt, color=c1] (1.4330127018922194,-0.25);
		\draw (1.4330127018922194,-0.25) edge[-,line width=2.5pt, color=c1] (1.0,-0.5);
		\draw (1.0,-0.5) edge[-,line width=2.5pt, color=c1] (0.5669872981077808,-0.25);
		\draw (0.5669872981077808,-0.25) edge[-,line width=2.5pt, color=c1] (0.5669872981077808,0.25);
		\draw (0.5669872981077808,0.25) edge[-,line width=2.5pt, color=c1] (1.0,0.0);
		\draw (1.0,0.0) edge[-,line width=1.5pt, color=c2] (1.4330127018922194,0.25);
		\draw (1.4330127018922194,0.25) edge[-,line width=1.5pt, color=c2] (1.4330127018922192,0.75);
		\draw (1.4330127018922192,0.75) edge[-,line width=1.5pt, color=c2] (1.0,1.0);
		\draw (1.0,1.0) edge[-,line width=1.5pt, color=c2] (1.0,0.5);
		\draw (1.0,0.5) edge[-,line width=1.5pt, color=c2] (1.0,0.0);
		\draw (1.0,0.0) edge[-,line width=1.5pt, color=c3] (0.5669872981077808,0.25);
		\draw (0.5669872981077808,0.25) edge[-,line width=1.5pt, color=c3] (0.1339745962155614,0.5);
		\draw (0.1339745962155614,0.5) edge[-,line width=1.5pt, color=c3] (0.5669872981077807,0.75);
		\draw (0.5669872981077807,0.75) edge[-,line width=1.5pt, color=c3] (1.0,0.5);
		\draw (1.0,0.5) edge[-,line width=1.5pt, color=c3] (1.0,0.0);
		\draw (1.0,0.0) edge[-,line width=1.5pt, color=c4] (0.5669872981077808,-0.25);
		\draw (0.5669872981077808,-0.25) edge[-,line width=1.5pt, color=c4] (0.13397459621556151,-0.5);
		\draw (0.13397459621556151,-0.5) edge[-,line width=1.5pt, color=c4] (0.5669872981077808,-0.75);
		\draw (0.5669872981077808,-0.75) edge[-,line width=1.5pt, color=c4] (1.0,-0.5);
		\draw (1.0,-0.5) edge[-,line width=1.5pt, color=c4] (1.0,0.0);
		\draw (1.0,0.0) edge[-,line width=1.5pt, color=c5] (1.0,-0.5);
		\draw (1.0,-0.5) edge[-,line width=1.5pt, color=c5] (1.0000000000000002,-1.0);
		\draw (1.0000000000000002,-1.0) edge[-,line width=1.5pt, color=c5] (1.4330127018922194,-0.75);
		\draw (1.4330127018922194,-0.75) edge[-,line width=1.5pt, color=c5] (1.4330127018922194,-0.25);
		\draw (1.4330127018922194,-0.25) edge[-,line width=1.5pt, color=c5] (1.0,0.0);
		\draw (1.0,0.0) edge[-,line width=1.5pt, color=c6] (1.4330127018922194,-0.25);
		\draw (1.4330127018922194,-0.25) edge[-,line width=1.5pt, color=c6] (1.8660254037844388,-0.5);
		\draw (1.8660254037844388,-0.5) edge[-,line width=1.5pt, color=c6] (1.8660254037844388,0.0);
		\draw (1.8660254037844388,0.0) edge[-,line width=1.5pt, color=c6] (1.4330127018922194,0.25);
		\draw (1.4330127018922194,0.25) edge[-,line width=1.5pt, color=c6] (1.0,0.0);
		\node[circle,fill=black,minimum size = 5pt, inner sep=0pt] () at (1.0,1.0) {};
		\node[circle,fill=black,minimum size = 5pt, inner sep=0pt] () at (1.0,0.5) {};
		\node[circle,fill=black,minimum size = 5pt, inner sep=0pt] () at (1.0,0.0) {};
		\node[circle,fill=black,minimum size = 5pt, inner sep=0pt] () at (1.0,-0.5) {};
		\node[circle,fill=black,minimum size = 5pt, inner sep=0pt] () at (1.0000000000000002,-1.0) {};
		\node[circle,fill=black,minimum size = 5pt, inner sep=0pt] () at (1.4330127018922192,0.75) {};
		\node[circle,fill=black,minimum size = 5pt, inner sep=0pt] () at (1.4330127018922194,0.25) {};
		\node[circle,fill=black,minimum size = 5pt, inner sep=0pt] () at (1.4330127018922194,-0.25) {};
		\node[circle,fill=black,minimum size = 5pt, inner sep=0pt] () at (1.4330127018922194,-0.75) {};
		\node[circle,fill=black,minimum size = 5pt, inner sep=0pt] () at (0.5669872981077807,0.75) {};
		\node[circle,fill=black,minimum size = 5pt, inner sep=0pt] () at (0.5669872981077808,0.25) {};
		\node[circle,fill=black,minimum size = 5pt, inner sep=0pt] () at (0.5669872981077808,-0.25) {};
		\node[circle,fill=black,minimum size = 5pt, inner sep=0pt] () at (0.5669872981077808,-0.75) {};
		\node[circle,fill=black,minimum size = 5pt, inner sep=0pt] () at (1.8660254037844386,0.5) {};
		\node[circle,fill=black,minimum size = 5pt, inner sep=0pt] () at (1.8660254037844388,0.0) {};
		\node[circle,fill=black,minimum size = 5pt, inner sep=0pt] () at (1.8660254037844388,-0.5) {};
		\node[circle,fill=black,minimum size = 5pt, inner sep=0pt] () at (0.1339745962155614,0.5) {};
		\node[circle,fill=black,minimum size = 5pt, inner sep=0pt] () at (0.13397459621556151,0.0) {};
		\node[circle,fill=black,minimum size = 5pt, inner sep=0pt] () at (0.13397459621556151,-0.5) {};
	\end{tikzpicture}}%
\caption{(a) Example of a team surviving orienteers problem with depot in the center. Thick edges correspond to survival probability 0.98, light edges have survival probability 0.91.  (b) Optimal paths for return constraint $p_s = 0.70$ and $K=6$. (c) Greedy paths for the same constraints. 
}
\end{figure}